\documentclass[letterpaper]{article}
\usepackage{flairs}
\usepackage{times}
\usepackage{helvet}
\usepackage{courier}
\usepackage{amsmath}
\usepackage{amssymb}
\usepackage{amsthm}
\usepackage{subcaption}
\usepackage{bbm}
\usepackage{xcolor}
\usepackage{url}
\usepackage{graphicx}

\newcommand{\adnan}[1]{{\color{red}#1}}
\newcommand{\haiying}[1]{{#1}}

\newcommand\shrink[1]{}

\def\n(#1){\bar{#1}}

\def\pr{{\it Pr}}
\def\A{{\bf A}}

\def\E{{\bf E}}
\def\e{{\bf e}}

\def\U{{\bf U}}
\def\u{{\bf u}}
\def\V{{\bf V}}
\def\v{{\bf v}}

\def\W{{\bf W}}
\def\w{{\bf w}}
\def\X{{\bf X}}
\def\x{{\bf x}}
\def\Y{{\bf Y}}
\def\y{{\bf y}}

\def\cond{\,|\,}

\def\AC{\mathcal{AC}}

\def\bTheta{\boldsymbol{\Theta}}
\def\eql(#1,#2){{#1\!\!=\!#2}}
\def\NP{\text{NP}}
\def\PP{\text{PP}}

\def\child{\text{ch}}

\def\defeq{\triangleq}

\newtheorem{theorem}{Theorem}

\newtheorem{definition}{Definition}

\newtheorem{proposition}{Proposition}

\DeclareMathOperator{\vars}{vars}
\DeclareMathOperator{\val}{val}

\DeclareMathOperator{\map}{MAP}
\DeclareMathOperator{\rmap}{RMAP}

\DeclareMathOperator{\dvar}{dvar}

\def\eql(#1,#2){{#1\!=\!#2}}

\newcommand{\argmax}{\operatornamewithlimits{arg max}}

\newcommand{\one}[1]{\mathbbm{1}_{#1}}

\def\clap#1{\hbox to 0pt{\hss#1\hss}}

\begin{document}
\title{Causal Unit Selection using Tractable Arithmetic Circuits}
\author{Haiying Huang, Adnan Darwiche\\
Computer Science Department\\
University of California, Los Angeles\\
\{hhaiying, darwiche\}@cs.ucla.edu\\
}
\maketitle

\begin{abstract}
The unit selection problem aims to find objects, called units, that optimize a causal objective function which describes the objects' behavior in a causal context (e.g., selecting customers who are about to churn but would most likely change their mind if encouraged). 
While early studies focused mainly on bounding a specific class of counterfactual objective functions using data, more recent work allows one to find optimal units exactly by reducing the causal objective to a classical objective on a meta-model, and then applying a variant of the classical Variable Elimination (VE) algorithm to the meta-model---assuming a fully specified causal model is available. In practice, however, finding optimal units using this approach can be very expensive because the used VE algorithm must be exponential in the constrained treewidth of the meta-model, which is larger and denser than the original model. We address this computational challenge by introducing a new approach for unit selection that is not necessarily limited by the constrained treewidth. This is done through compiling the meta-model into a special class of tractable arithmetic circuits that allows the computation of optimal units in time linear in the circuit size.
We finally present empirical results on random causal models that show order-of-magnitude speedups based on the proposed method for solving unit selection.
\end{abstract}

\section{Introduction}
A theory of causality has emerged over the last few decades based on two parallel hierarchies, 
an {\em information hierarchy} and a {\em reasoning hierarchy,} often called the {\em causal hierarchy}~\cite{pearl18,Bareinboim20211OP}. 
On the reasoning side, this theory has crystallized three levels of reasoning with increased sophistication and proximity 
to human reasoning: associational, interventional and counterfactual, which are exemplified by the following canonical probabilities.
{\em Associational} \(\pr(y | x)\): probability of \(y\) given that \(x\) was observed (e.g., probability that a patient has a flu given they have a fever).
{\em Interventional} \(\pr(y_x)\): probability of \(y\) given 
that \(x\) was established by an intervention,
which is different from \(\pr(y | x)\) (e.g.,
seeing the barometer 
fall tells us about the weather but moving the barometer needle won't bring rain).
{\em Counterfactual} \(\pr(y_x | y', x')\): probability of \(y\) if we were to establish
\(x\) by an intervention given that neither \(x\) nor \(y\) are true (e.g., probability that a patient who
did not take a vaccine and died would have recovered had they been vaccinated).
On the information side, these forms of reasoning 
require different
levels of knowledge, encoded as associational, causal and functional (mechanistic) models, 
with each class of models containing more information than the preceding one. In the framework of probabilistic graphical models
\cite{PGMbook}, such knowledge is encoded 
by Bayesian networks~\cite{Pearl88b,DarwicheBook09}, 
causal Bayesian networks~\cite{pearl00b,PetersBook,SpirtesBook} and functional Bayesian networks~\cite{uai/BalkeP95}
also called {\em structural causal models} (SCMs)~\cite{halpern2000axiomatizing}.

The \textit{unit selection problem} emerged from this theory of causality and has been receiving attention as it provides a rigorous framework to analyze and optimize the causal behavior of objects (e.g., agents). 
It was initially proposed by~\cite{ijcai/LiP19} who motivated it using the problem of selecting customers to target by an encouragement offer for renewing a subscription. 
The unit selection problem is based on a \textit{causal objective function} which mentions a special set of variables $\U$, called \textit{unit variables,} and can include any quantities from the causal hierarchy: associational, interventional, or counterfactual. The goal of this problem is to find an instantiation $\u$ of unit variables which optimizes the given objective function. Intuitively, this is meant to identify units ($\u$), which can be individuals or arbitrary objects such as policies, based on their causal behavior. 

Unit selection embeds 
two subproblems which are hard in general: {\em evaluation} and {\em optimization} of the objective function~\cite{ijcai/LiP19}. Prior work focused on the evaluation problem, which is concerned with estimating the value of a causal objective function using observational and experimental data given a non-parametric causal graph. This stems from the identifiability problem, which has been extensively studied by the causality community. Useful bounds and criteria have been derived for this purpose with respect to a specific causal objective function called the \textit{benefit function} ~\cite{ijcai/LiP19,li2022unit_ab,aaai/LiP22,corr/LiP22a,li2022unit}.

A recent work~\cite{huang2023algorithm} investigated the optimization problem: finding units that maximize the score defined by the causal objective function---assuming a fully specified SCM so as to obtain point-values for the objective function. They first show that a broad class of causal objective functions can be reduced to an associational probability on a meta-model, called the \textit{objective model.} They
also showed that optimizing this (surrogate) probability requires solving a variant of the Maximum A Posteriori (MAP) inference problem, called \textit{Reverse-MAP (R-MAP).} They further showed how to modify the classical Variable Elimination (VE) algorithm to solve R-MAP. This treatment, in principle, would allow us to find the exact optimal units. However, computing R-MAP on the objective model remains a hard task---in fact, it is $\NP^{\PP}$-complete as shown in~\cite{huang2023algorithm}. In practice, the objective model is usually much larger and denser than the original SCM. Moreover, running VE on the objective model must be exponential in its constrained treewidth~\cite{jair/ParkD04}, and can 
also be exponential in the number of unit variables depending on the form of the specific objective function. 

We address this computational challenge by introducing a new method for solving R-MAP on a model $G$ by compiling $G$ into a tractable arithmetic circuit (AC)~\cite{Darwiche03,faia/Darwiche21} and computing R-MAP on the circuit in linear time. Using ACs for inference allows us to exploit the local/parameteric structure of a model which means that we may be able to handle models that are out of reach of classical, structure-based inference algorithms like VE; see, e.g.,~\cite{Chavira.Darwiche.Ijcai.2005,ChaviraDJ06,Chavira.Darwiche.Aij.2008}. And this is indeed what our empirical evaluation shows.

\shrink{
In particular, we show that if an arithmetic circuit (AC) that represents a distribution $\pr$ satisfies certain constraints known as {\em marginal determinism}, then the circuit can be modified to compute R-MAP probabilities with respect to $\pr$ in time linear in the circuit size, similar to how variable elimination for MAP is modified to compute R-MAP.}

This paper is structured as follows. We start by a review of the unit selection problem and its reduction to R-MAP. We then provide some background on ACs and how they can be used to solve MAP in time linear in the AC size. We follow by our proposed linear-time method for solving R-MAP on ACs, and a corresponding empirical evaluation which shows orders-of-magnitude speedups compared to the state-of-the-art. We finally close with some concluding remarks.

\section{Unit Selection by Reduction to Reverse-MAP}\label{sec:unit-select} 

A causal objective function is an expression that involves associational (observational), interventional or counterfactual probabilities. We begin by formulating the selection of customers to target by an advertisement using a specific causal objective function called the benefit function. We next review the reduction from unit selection to Reverse-MAP using an objective model. We finally review the computation of Reverse-MAP using variable elimination.

\begin{figure*}[t]
    \centering
\begin{subfigure}{0.1\textwidth}
    \includegraphics[width=\textwidth]{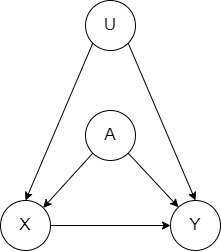}
    \caption{SCM}
    \label{fig:scm}
\end{subfigure}
\hspace{1cm} 
\begin{subfigure}[b]{0.2\textwidth}
    \includegraphics[width=\textwidth]{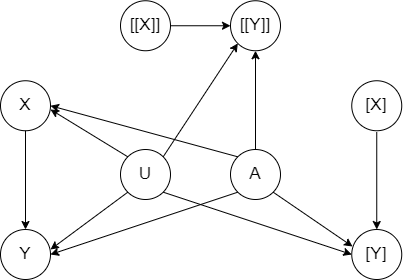}
    \caption{triplet model (mutilated)}
    \label{fig:triplet}
\end{subfigure}
\hspace{1cm} 
\begin{subfigure}[b]{0.42\textwidth}
    \includegraphics[width=\textwidth]{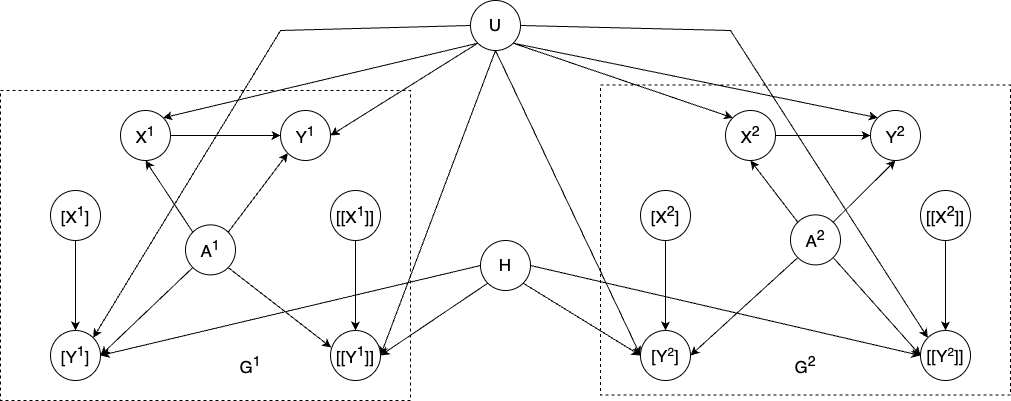}
    \caption{objective model}
    \label{fig:objective}
\end{subfigure}

\caption{Example showing the objective model for a causal objective function with 2 components (Huang and Darwiche 2023).}
\label{fig:test}
\end{figure*}

Consider a company that wants to identify customers most likely to respond positively to a pop-up ad of a new product. Suppose the product profit is \$50 and the cost of placing the ad to a user is \$10. Let variable $Y$ denote whether the user purchases the product (outcome) and variable $X$ denote whether they are targeted by the ad (action).  Let variables $\U$ be the observed customer characteristics (gender, area, purchase history, etc) that affect both $X$ and $Y,$ and variables $\A$ be the hidden factors that affect $X$ and $Y$. This causal model is shown in Figure~\ref{fig:scm}. We can partition customers into four types using counterfactuals.
A responder (\(y_x, y_{x'}'\)) purchases the product if targeted but does not otherwise, so the benefit of selecting a responder is $\beta = 50 - 10 = 40$. An always-taker (\(y_x,y_{x'}\)) purchases regardless of advertisement, so its benefit is $\gamma = -10$ (the ad cost is lost). An always-denier (\(y_x',y_{x'}'\)) does not purchase regardless of advertisement, so its benefit is $\theta = -10$. A contrarian (\(y_{x'},y_x'\)) does not purchase if targeted but purchases otherwise, so its benefit is $\delta = -50 - 10 = -60$. The company can then select customers $\u$ that maximize the  benefit function~\cite{ijcai/LiP19,aaai/LiP22,corr/LiP22a}:
\begin{align}
     L(\u) &= 
     \beta \pr(y_x,  y_{x'}'|\u) +  
     \gamma\pr(y_x,  y_{x'}|\u) \nonumber \\ 
     &\quad +\theta\pr(y_x', y_{x'}'|\u) + 
     \delta\pr(y_x', y_{x'}|\u)
     \label{eqn:benefit-ang}
\end{align}
\shrink{
Here, $(\beta, \gamma, \theta, \delta)$ is the benefit of selecting one individual of each response type and $L(\u)$.}

This setup was generalized in~\cite{huang2023algorithm} to (almost) any causal objective function that is a linear combination of counterfactual, interventional, or observational probabilities, which has the following form:
\begin{equation}
     L(\u) = \sum_{i=1}^n w_i\cdot\pr(\y^i_{\x^i},\w^i_{\v^i}|\e^i, \u) \quad 
    \label{eqn:objective-function}
\end{equation}
This class of causal objective functions includes the function in Equation~\ref{eqn:benefit-ang} but is more expressive. 
First, it allows variable sets $\Y^i, \W^i, \X^i, \V^i$  for representing compound treatments and outcomes instead of a single action (ad recommendation) and outcome (purchase). It also allows any combination of their values in each component $\pr^i$. For example, it can include components like $\pr((\bar{t}, b)_{(x,e)}, (t, \bar{b})_{(\bar{x}, e)} \cond \u)$: probability that one would have normal temperature and normal blood pressure $(\bar{t}, b)$ had they taken medicine and exercised $(x, e)$, and would have high temperature and low blood pressure $(t, \bar{b})$ had they only exercised $(\bar{x}, e)$. 
Second, it allows evidence $\E^i$ to incorporate what happens in reality. For example, the probability $\pr(y'_x \cond y, x', \u)$ that one would not have gotten Covid $(y')$ if they were to be vaccinated $(x)$, given that they did not and contracted Covid $(y, x')$. 
Third, it supports components $\pr^i$ with different treatment and outcome variables, which means that each component of the function can consider different aspects of the problem with different weights. For example, assessing the causal effect of an educational policy that applies $x$ and evaluates students' SAT score $y$ in the first year; and applies $v$ and evaluates their extracurricular activity $w$ in the second year; e.g., $L(\u) = 0.7\pr(y_{x}|\u) + 0.3\pr(w_{v}|\u)$. 
Finally, we are particularly interested in structured units (e.g., decisions, policies, people, situations,
regions, activities) that correspond to many unit variables, so that finding the optimal units poses a scalability problem. 
We next review the construction of the objective model in (Huang and Darwiche 2023) which reduces this type of objective function to a single observational probability in a principled way.

\paragraph{Objective Model} Consider an SCM $G$ (Figure~\ref{fig:scm}) with unit variables $\U$. The first step is to reduce each counterfactual component \(\pr({\y^i}_{\x^i},{\w^i}_{\v^i}|\e^i, \u)\) to an observational probability on a triplet model $G^i$ (see Figure~\ref{fig:triplet}). This is done by creating three copies of $G$, denoted by $G$, $[G]$, $[[G]]$ and joining them so that they share all \haiying{exogenous (root)} variables.\footnote{We need three copies as $G$ is for the real-world, $[G]$ is for the event ${\y^i}_{\x^i}$, and $[[G]]$ is for the event ${\w^i}_{\v^i}$.} We then remove edges pointing into treatment variables $\X^i$, $\V^i$ in copies $[G]$ and $[[G]]$ respectively. The counterfactual probability is then reduced to an observational one, $\pr^i([\y^i], [[\w^i]] \cond [\x^i], [[\v^i]], \e^i, \u),$ on $G^i$. We point out that this is based on a standard technique known as twin-networks~\cite{aaai/BalkeP94}. 

Given a list of triplet models $G^1,\ldots,G^n$, the next step is to combine $(G^i, \pr^i)$ such that the linear combination $\sum \pr^i$ is reduced to a single observational probability. As shown in~\cite{huang2023algorithm}, this can be achieved by adding an auxiliary node $H$ as a common parent of outcome nodes and constraining unit variables $\U$ among $G^i$ to take the same value: \haiying{if a variable $U\in\U$ is a root, then it is shared among $G^i$; otherwise an auxiliary node is added as a common child for copies of $U$ among $G^i$}. Node $H$ is called a mixture node as it is responsible for inducing a mixture $\sum \pr^i$.
For example, consider an objective function with two components $L(u) = w_1\cdot\pr(y_x, y'_{x'}|u) + w_2\cdot\pr(y_x, y_{x'}|u)$. 
Figure~\ref{fig:objective} shows the corresponding objective model $G'$. We can then evaluate $L$ as a classical probability $\pr'([y^1], [[{y'}^1]], [y^2], [[y^2]] \mid [x^1], [[{x'}^1]], [x^2], [[{x'}^2]],u)$ on $G'$.
The correctness of the objective model has been studied in-depth~\cite{huang2023algorithm} and is not the focus of this work. We finally highlight that the problem now amounts to optimizing a classical probability of the form $\max_{\u} \pr(\e_1|\u, \e_2)$ on the objective model, which can be larger and denser than the original SCM.

\paragraph{Reverse-MAP (R-MAP)} Consider an SCM $G$ and let $\U$, $\E_1$, $\E_2$ be disjoint sets of variables in $G$. The R-MAP problem is defined as \(\rmap(\U, \e_1, \e_2) \triangleq \argmax_{\u} \pr(\e_1|\u, \e_2)\) --- to be contrasted with the classical MAP problem, $\map(\U, \e) \triangleq \argmax_\u \pr(\u|\e) = \argmax_\u\pr(\u, \e)$\haiying{~\cite{dechter2003mini,jair/ParkD04,DarwicheBook09}}. In the context of causal reasoning, MAP can be thought of as finding the most likely outcome given a specific cause, while R-MAP is concerned with finding the most likely cause for a specific outcome (given some other information). They are different problems as discussed in~\cite{huang2023algorithm} and R-MAP captures the essence of unit selection better than MAP.

\cite{huang2023algorithm} showed how to adapt the classical variable elimination (VE) algorithm to solve R-MAP. Let $\V$ be the non-MAP variables. The basic idea is to apply two parallel runs of VE in a synced way. In the first run, we compute $\pr(\U, \e_1\e_2)$ by summing out $\V$ under evidence $\e_1\e_2$. In the second run, we compute $\pr(\U, \e_2)$ by summing out $\V$ under evidence $\e_2$. As shown in~\cite{huang2023algorithm}, we can efficiently ``divide'' $\pr(\U, \e_1\e_2)$ and $\pr(\U, \e_2)$ to obtain $\pr(\e_1|\U,\e_2)$ if the two runs use the same elimination order.\footnote{The difficulty is that the space of $\U$ can be very large.} We finally compute the R-MAP probability by maxing out variables $\U$.

\section{Basics of Arithmetic Circuits}
\label{sec:basics}
\begin{figure}[tb]
\centering
\small
\begin{tabular}[b]{cc|c}
\(A\) & \(B\) & \(f\) \\\hline
$a$ & $b$ & $3$ \\
$a$ & $\bar{b}$ & $4$ \\
$\bar{a}$ & $b$ & $10$ \\
$\bar{a}$ & $\bar{b}$ & $12$
\end{tabular}
\includegraphics[width=.50\linewidth]{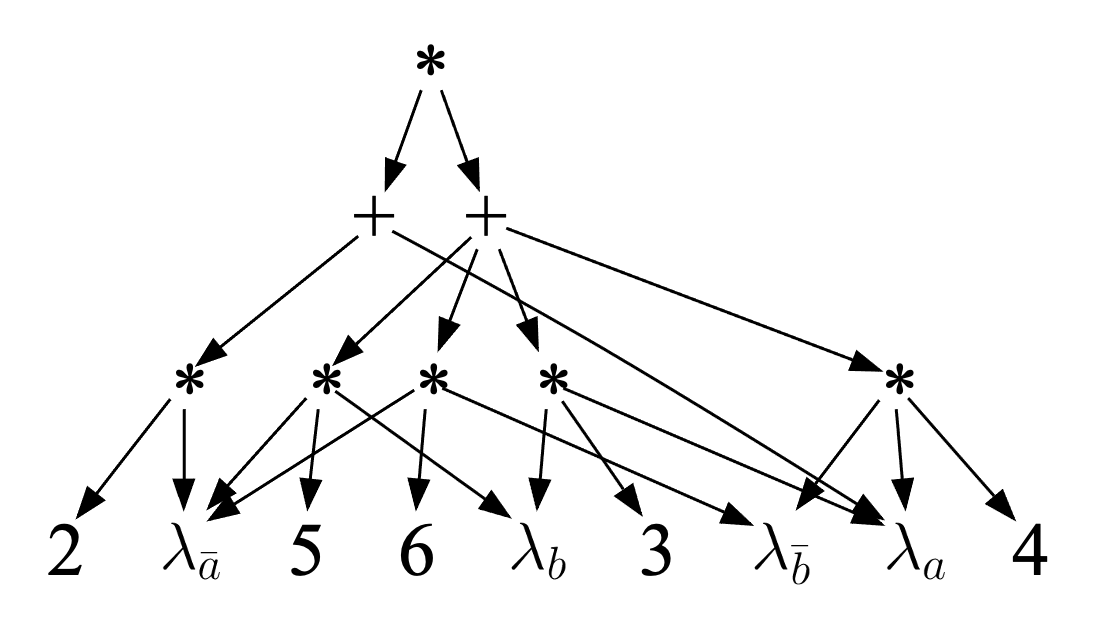}
\caption{An AC that computes factor \(f(A,B)\).}
\label{adnan:fig:AC1-def}
\end{figure}

An arithmetic circuit (AC) is based on a set of {\em discrete variables,} which define a key ingredient of 
the circuit: the {\em indicators.} For each value \(x\) of a variable \(X\), we have an indicator \(\lambda_x\). 
The AC will then have constants (parameters)
and indicators (inputs) as its leaf nodes with adders and multipliers as its internal nodes; see Figure~\ref{adnan:fig:AC1-def}.
An AC represents a factor which is a mapping from variable instantiations to non-negative numbers. 
A probability distribution is a special case of a factor so an AC can represent a distribution too, which is our focus in this work.
The factor represented by an AC is obtained by evaluating the AC at complete variable instantiations.
To evaluate the AC at an instantiation \(\e\), we replace each indicator \(\lambda_x\) 
with \(1\) if the value \(x\) is compatible with instantiation \(\e\) and with \(0\) otherwise~\cite{Darwiche03}. 
We then evaluate the AC bottom-up in the standard way.
The factor \(f(A,B)\) in Figure~\ref{adnan:fig:AC1-def} has four rows which correspond to the four instantiations of variables \(A\) and \(B\). 
Evaluating the AC in this figure at each of these complete instantiations 
yields a value for each instantiation and therefore defines its factor.
We say in this case that the AC {\em computes} this factor.
\shrink{We also say that this circuit
can {\em lookup} the values of this factor. 
For example, in Figure~\ref{adnan:fig:AC1-eval} (left), the circuit evaluates to \(12\) under the complete variable
instantiation \(\adnaneql(A,{\adnann{a}}),\adnaneql(B,{\adnann{b}})\) by setting the indicators to
\(\adnaneql({\lambda_a},0), \adnaneql({\lambda_{\adnann{a}}},1), \adnaneql({\lambda_b},0), \adnaneql({\lambda_{\adnann{b}}},1)\). 
}
An AC can be evaluated at a {\em partial} variable 
instantiation using the same procedure, but the value returned may not be meaningful unless the circuit satisfies certain properties. Three key such properties are
decomposability, determinism and smoothness~\cite{Darwiche03}.
\haiying{Let $\vars(n)$ be the subset of variables whose indicators appear at or below a node $n$.
Decomposability requires that $\vars(n_1) \cap \vars(n_2) = \emptyset$ for every two children $n_1$ and $n_2$ of a $*$-node. Smoothness requires that  $\vars(n_1) = \vars(n_2)$ for every two children $n_1$ and $n_2$ of a $+$-node.}
Determinism requires at most one non-zero 
child for each $+$-node, when the circuit is evaluated under any complete variable instantiation.
An AC that represents a probability distribution $\Pr(\X)$
will return the marginal $\Pr(\e)$ when evaluated at input $\e,$
assuming the AC is deterministic, 
decomposable and smooth~\cite{Darwiche03}. 
It will also compute the MPE probability under evidence $\e$
in this case, assuming we replace $+$-nodes with $\max$-nodes~\cite{ChanD06}. In fact, determinism is not needed for
computing marginals as shown initially in~\cite{PoonD11} and discussed in detail in~\cite{ChoiDarwiche17}. See also~\cite{faia/Darwiche21} for a recent tutorial/survey on ACs and their properties.

\section{Solving MAP using Arithmetic Circuits}\label{sec:AC}

We next discuss a special class of ACs that we shall utilize in this work: 
decision-ACs. 
\cite{huang2006solving} showed how
decision-ACs can be used to solve MAP in time linear in the circuit size, assuming they are
constructed subject to specific constraints. 
\cite{pipatsrisawat2009new} built
on these observations to show similar results with respect to a problem related to MAP.
\cite{ChoiDarwiche17} identified a general condition (determinism after projection) that these constraints ensure, which allows MAP to be solved in linear time.\shrink{\footnote{This condition was later referred to as marginal determinism in~\cite{choi2022solving}. Marginalization is the counterpart of projection: to project on a set of variables is to marginalize the complement of these variables.}}
We review these earlier findings next and formalize some of the associated observations as we need them to 
provide a basis for our treatment for R-MAP.

\shrink{
\adnan{Kisa and Choi has a paper on compiling BNs into SDDs then AC (around 2014). This is relevant to cite as well, later.}
}

\begin{definition}[Decision-AC]
    Let $\AC(\X)$ be a decomposable and smooth arithmetic circuit over variables $\X$. $\AC$ satisfies the \underline{decision property} if every $+$-node has the form $n = \sum_{i} \lambda_{x_i} * n_i$  where $x_i$ are distinct values of some variable $X \in \X$ and $n_i$ are circuit nodes. We say $X$ is the decision variable of node $n$, denoted $\dvar(n)$.
    \label{def:decision-ac}
\end{definition}

\shrink{Consider a $f(X, \Y)$ and assume $X$ is binary. It can always be decomposed as $f(X, \Y) = \one{x} * g_1(\Y) + \one{\bar{x}} * g_2(\Y)$.}
The decision property implies determinism so decision-ACs are decomposable, smooth and deterministic. Decision-ACs are the numerical analog of the Boolean decision-DNNFs~\cite{huang2007language} and have been used extensively in probabilistic reasoning~\cite{Darwiche03,huang2006solving,Chavira.Darwiche.Ijcai.2007,Chavira.Darwiche.Aij.2008,ChaviraDJ06,pipatsrisawat2009new}.\footnote{For example, the state-of-the-art {\tt ACE} system, \url{http://reasoning.cs.ucla.edu/ace}, encodes a Bayesian Network using a CNF, compiles the CNF into a decision-DNNF using the {\tt C2D}~\cite{darwiche2004new} compiler, and finally converts the decision-DNNF into a decision-AC. Compilation based on variable elimination also produces decision-ACs~\cite{DarwicheBook09}.}

\begin{proposition}
    Consider a decision-AC over variables $\X$ and let $\U \subseteq \X$. If the circuit satisfies: 1) no $+$-node $n$ with $\dvar(n) \in \U$ is below some $+$-node $m$ with $\dvar(m) \notin \U$ and 2) every indicator $\lambda_x$ is only attached to some $+$-node $n$ with $\dvar(n)=X$, then this circuit supports linear-time MAP over variables $\U$.
    \label{prop:decision-ac-constraint}
\end{proposition}

\shrink{
Let $\U$ be MAP variables and $\V$ be other variables. If a decision-$\AC$ is constrained such that no $+$-node with $\dvar \in \U$ can appear below a $+$-node with $\dvar \in \V$, then we can compute the MAP probability $\map_p(\U, \e)$ by evaluating the $\AC$ bottom-up under evidence $\e$ and replacing every $+$-node whose $\dvar \in \U$ with a $\max$-node. This condition can be enforced by using a constrained variable order with $\U$ appearing in the last when constructing this decision-$\AC$.
}

If a decision-AC satisfies Conditions~(1) and (2) above, then we can compute $\map(\U, \e)$ exactly by traversing the AC bottom-up, while replacing every $+$-node $n$ with decision variable $\dvar(n)\in\U$ with a $\max$-node. This was first claimed in~\cite{huang2006solving} without a formal proof and was used in later works to solve related problems, e.g.,~\cite{pipatsrisawat2009new,choi2020probabilistic}.
\cite{huang2006solving} identified Condition~(1) of
Proposition~\ref{prop:decision-ac-constraint} but Condition~(2) was never made explicit in the literature as far as we know but is needed for the proof of Proposition~\ref{prop:decision-ac-constraint}.\footnote{{\tt ACE} compiles decision-ACs
that satisfy Condition~(2).} 
\shrink{
Since Decision-AC allows a mixture of decision nodes and conjunction nodes\footnote{This separates decision-DNNF from FBDD~\cite{huang2007language}.}, it is necessary to ensure indicators only attached to decision nodes to ensure the correctness of linear-time MAP. \haiying{For a counterexample, see appendix ???.}
}
Our proof of this proposition considers a more general condition identified in~\cite{ChoiDarwiche17} which allows linear-time MAP on ACs, and then shows that decision-ACs satisfy this condition; see Proposition~\ref{thm:decision-ac} which is stated later. 

Let $\U$ be the MAP variables and $\V$ be all other variables.
The MAP problem can be solved by first computing the marginal $\pr(\U, \e) = \sum_\V \pr(\X, \e),$ which sums-out variables $\V$ (projects on variables $\U$), and then computing the MPE, $\max_\U \pr(\U, \e)$; see, e.g.,~\cite{DarwicheBook09}. We know that a decomposable and smooth AC supports linear-time marginals
as we can simply set all indicators of variables $\V$ to $1$.
If the AC is also deterministic, then it also supports linear-time MPE. However, after projecting a deterministic, decomposable, and smooth AC on variables $\U,$ the resulting AC over variables $\U$ may no longer be deterministic. If an AC remains deterministic after being projected on $\U$, we can easily compute MAP by evaluating the circuit bottom-up while replacing every $+$-node that {\em depends on variables in $\U$} with a $\max$-node. This was first shown in~\cite{ChoiDarwiche17} 
and the determinism-after-projection property was 
later referred to as {\em marginal determinism} in ~\cite{choi2022solving}.

\begin{definition}
    Consider a decomposable and smooth AC over variables $\X$ and let $\U \subseteq \X$. The AC is \underline{$\U$-deterministic} iff the following holds: for any $+$-node $n$, if $n$ depends on $\U$ (i.e., $\vars(n) \cap \U \neq \emptyset$), then at most one child of $n$ can be non-zero when the AC is evaluated at any input $\u$.
    \label{def:mar-determinism}
\end{definition}

\shrink{
\adnan{Determinism does not imply $\U$-determinism, and vice versa; see Appendix~?? for a concrete example that illustrates this.}
}

\shrink{To distinguish $\U$-determinism from the general determinism defined in Section \ref{sec:AC}, we may sometimes refer to determinism as total determinism when necessary.}

If an AC is $\U$-deterministic, then it can be used to compute the MAP probability $\map(\U, \e)$ under any evidence $\e$ by performing a simple, bottom-up traversal as below:
{\footnotesize
\[
\val(n) =
\begin{cases}
1, & \text{if $n=\lambda_x$ \& $x\land\e$ is consistent} \\
0, & \text{if $n=\lambda_x$ \& $x\land\e$ is inconsistent} \\
p, & \text{if $n$ is a parameter $p$} \\
\prod_{c \in \child(n)} \val(c) & \text{if $n$ is a product node} \\
\max_{c \in \child(n)} \val(c) & \text{if $n$ is a sum node} \\
                               & \text{and $\vars(n) \cap \U = \emptyset$} \\
\sum_{c \in \child(n)} \val(c), & \text{if $n$ is a sum node} \\
                                & \text{and $\vars(n) \cap \U \neq \emptyset$}
\end{cases}
\]
}

The correctness of this procedure is established as follows. Given decomposability and smoothness, by fixing the indicators of variables $\V$ to 1, we obtain another decomposable and smooth AC that computes the projection $\pr(\U) = \sum_\V \pr(\U, \V)$. We can then reduce nodes that do not depend on $\U$ to constants (parameters). This leads to a projected AC that depends only on $\U$, $\AC_p(\U)$.
If $\AC(\X)$ is $\U$-deterministic, then $\AC_p(\U)$ must be deterministic, so it can compute the
MPE probability $\max_\U\pr(\U)$ in linear time after replacing its $+$-nodes with $\max$-nodes.

We now have the following result which immediately implies
Proposition~\ref{prop:decision-ac-constraint} and, hence, shows that decision-ACs with appropriate constraints support linear-time MAP.

\shrink{
\begin{proposition}
    Let $\AC$ be a decomposable and smooth circuit over variables $\X$ and let $\U \subseteq \X$. If $\AC$ is $\U$-deterministic, then it supports MAP with target variables $\U$ in linear time in the circuit size.
    \label{thm:map}
\end{proposition}
}

\begin{proposition}
A decision-AC is $\U$-deterministic if it
satisfies the two conditions of Proposition~\ref{prop:decision-ac-constraint}.
    \label{thm:decision-ac}
\end{proposition}

\begin{proof}
Suppose the decision-AC is evaluated at input $\u^*$
which is some instantiation of $\U.$
\shrink{
    Let $\u^*$ be any instantiation of $\U$. Suppose the decision-AC is evaluated at inputs $\x \sim \u^*$.} Consider any $+$-node $n$ with decision variable $X$. If $X \in \U$, it is easy to verify that at most one child of $n$ can be nonzero---the one with indicator $\lambda_{x^*}$, where $x^*$ is the state of $X$ in $\u^*$. If $X \notin \U$, we claim that node $n$ cannot depend on $\U$ (i.e., $\vars(n) \cap \U = \emptyset$). Hence, this decision-AC is $\U$-deterministic by Definition~\ref{def:mar-determinism}. We prove this claim by contradiction. Suppose node $n$ does depend on some variable $U \in \U$, which means there is an indicator $\lambda_u$ below $n$. By Condition~(2) of Proposition~\ref{prop:decision-ac-constraint}, this indicator $\lambda_u$ must be attached to some $+$-nodes $m$ with decision variable $U$. This implies that $m$ is below $n$ which contradicts Condition~(1) of Proposition~\ref{prop:decision-ac-constraint}.
\end{proof}

\section{Circuits for Reverse-MAP}
We now consider the main question behind our proposed method for unit selection: under what condition, and why, will a circuit attain the ability to support efficient R-MAP? The answer is motivated by the following observation.

\paragraph{Circuit Division} The primitive operation required by R-MAP, beyond the existing ones for classical MAP, is the ability to divide two distributions that have the same domain. 
That is, 
to obtain $\pr(\e_1|\U, \e_2)$, we need to divide $\pr(\U, \e_1\e_2)$ and $\pr(\U, \e_2)$ which is generally hard. Hence, we raise the question: given two distributions $\pr_1(\X)$ and $\pr_2(\X)$ computed using ACs, can we efficiently obtain an AC that computes their quotient $\pr_3(\X) = \pr_1(\X) / \pr_2(\X)$? We show that this is feasible if $\pr_1(\X)$ and $\pr_2(\X)$ are computed by two ACs with the same structure (but with different parametrizations), assuming the ACs are deterministic, decomposable and smooth.

\begin{theorem}
Consider an AC that is 
deterministic, decomposable and smooth under both parametrization $\bTheta_1$ and $\bTheta_2$.
Suppose further the AC
computes distribution $\pr_1(\X)$ under $\bTheta_1$ and distribution $\pr_2(\X)$ under $\bTheta_2$. Then
the AC is deterministic, decomposable and smooth
and 
computes $\pr_3(\X) = \pr_1(\X) / \pr_2(\X)$
under parametrization
$\bTheta_3 = \{ \theta_1 / \theta_2: \theta_1,\theta_2 \text{ are corresponding parameters in } \bTheta_1,\bTheta_2\}$.\footnote{We assume $\pr_2$ has larger support than $\pr_1$, i.e., $\pr_2(\x) = 0$ only if $\pr_1(\x) = 0$. We define $0/0 \defeq 0$.}
    \label{thm:circuit-div}
\end{theorem}

Hence, we can divide two distributions $\pr_1$ and $\pr_2$ induced by the same AC by simply dividing the corresponding parameters in the ACs that lead to  $\pr_1$ and $\pr_2$. Decomposability and smoothness 
are not enough though, we also need determinism.
To prove Theorem~\ref{thm:circuit-div}, we need the notion of a subcircuit which is introduced in \cite{ChanD06} and studied extensively in~\cite{ChoiDarwiche17}.

\begin{definition}
    Let $\AC(\X)$ be a decomposable and smooth circuit. A \underline{complete subcircuit} $\alpha$ of $\AC$ is obtained by visiting the circuit nodes top-down starting at the root: if a $*$-node is visited, visit all its children, and if $+$-node is visited, visit exactly one of its children. The \underline{term} of $\alpha$ is the set of variable values appearing in indicators of $\alpha$, and the \underline{coefficient} of $\alpha$ is the product of all parameters in $\alpha$. 
    \label{def:subcircuit}
\end{definition}

A complete subcircuit must include exactly one indicator for every variable in $\X$. Hence, the term of each complete subcircuit corresponds to an instantiation $\x$ of $\X$, so the subcircuit is called an $\x$-subcircuit. Evaluating $\AC(\x)$ amounts to summing the coefficients of all $\x$-subcircuits. Furthermore, \cite{ChoiDarwiche17} showed that if the AC is also deterministic, then for any instantiation $\x$ of $\X$, if $\AC(\x)\neq 0$, there is a unique $\x$-subcircuit whose coefficient is nonzero. Denote this subcircuit by $\alpha_\x$. We know $\alpha_\x$ has coefficient $p_1 = \pr_1(\x)$ under $\bTheta_1$ and coefficient $p_2 = \pr_2(\x)$ under $\bTheta_2$. By construction, $\alpha_\x$ must have coefficient $p_3 = p_1 / p_2$ under $\bTheta_3$. Hence, $\AC(\x) = p_3 = \pr_3(\x)$ under $\bTheta_3$. This proves Theorem~\ref{thm:circuit-div}.

Without determinism, an AC may have multiple $\x$-subcircuits for a given $\x$, and dividing corresponding parameters may not give the correct result. For a simple example, consider $\AC(X) = a \cdot\lambda_x + b \cdot \lambda_{\bar{x}} + c \cdot \lambda_x$. $\AC$ is not deterministic; it has two complete subcircuits $(x, a)$ and $(x, c)$ for $x$. We have $\AC(x) = a_1 + b_1$ under $\bTheta_1$ and $\AC(x) = a_2 + b_2$ under $\bTheta_2$. After dividing corresponding parameters, $\AC(x)$ produces $a_1 / a_2 + b_1 / b_2$  instead of $(a_1 + b_1) / (a_2 + b_2)$.

\begin{table*}[tb]
\centering
\small
\begin{tabular}{|c|c|c||c|c|c|c||c|c|c|}
\hline
\multicolumn{3}{|c||}{SCM} & \multicolumn{4}{|c||}{VE}       & \multicolumn{3}{|c|}{ACE} \\
\hline
n  &  n'    & $|\U|$   & done  & time (s) & ve\_size   & tw  & done   & time (s)     & ac\_size  \\
\hline
10 & 18.9      & 3.0  & 25   & 0.49  & 4.92e5 & 12.28   &25  & 2.09   & 1595   \\
\hline
15 & 28.4      & 3.4  & 25   & 1.42  & 9.15e6 & 16.88  & 25 & 3.31   & 4990    \\
\hline
20 & 36.8      & 4.0  & 25   & 4.72  & 1.02e8 & 20.12 & 25 & 4.68   & 1.03e4    \\
\hline
25 & 45.5      & 5.3  & 25   & 45.90 & 2.42e9 & 24.16 & 25 & 5.76    & 3.90e4  \\
\hline
30 & 54.5      & 5.8  & 22   &145.72 & 8.07e9 & 27.36    & 25   & 13.62 & 1.12e5 \\
\hline
35 & 63.7      & 7.0  & 7    & 274.73 & 1.05e10 & 32.6   & 24   & 56.33 & 3.57e5  \\
\hline
40 & 74.0      & 10.1 & 0    & /     &  /     & 40.0   & 20   & 131.78  & 3.16e6  \\
\hline
\end{tabular}
\caption{ACE\_RMAP v.s. VE\_RMAP. Here, $n$ is \#nodes we start with to sample SCMs and $n'$ is \#nodes in the sampled SCM eventually. The ``done'' column is the number of solved instances out of $25$. Each data point is an average over solved instances.}
\label{tab:benefit-function-exp}
\end{table*}

\paragraph{R-MAP} We now introduce our treatment for solving R-MAP using ACs based on the techniques for MAP and division on ACs. Let $\U$ be the target variables, $\V$ be the non-target variables, and $\E_1,\E_2$ be evidence variables. Given a $\U$-deterministic $\AC(\X)$ that represents distribution $\pr(\X)$, we can compute the R-MAP probability $\rmap(\U, \e_1, \e_2)$ by running a two-pass traversal on the AC, similar to how variable elimination is extended to R-MAP.

Let $\AC_\U$ be the set of nodes in $\AC$ that depend on $\U$ (i.e., $\vars(n) \cap \U \neq \emptyset$), and $\AC_\V$ be the set of nodes that do not (i.e., $\vars(n) \cap \U = \emptyset$). In the first pass, we evaluate nodes in $\AC_\V$ bottom-up under input $\e_1\e_2$. This leads to a projected circuit $\AC_\U(\U)$ with parametrization $\bTheta_1$, which computes $\pr(\U, \e_1\e_2)$ and is deterministic. In the second pass, we evaluate $\AC_\V$ bottom-up under input $\e_2$. This leads to $\AC_\U(\U)$ with parametrization $\bTheta_2$, which computes $\pr(\U, e_2)$ and is deterministic. What we need is $\pr(\e_1\cond\U, \e_2) = \pr(\U, \e_1\e_2) / \pr(\U, \e_2)$. By Theorem~\ref{thm:circuit-div}, this is simply done by dividing corresponding parameters in $\bTheta_1$ and $\bTheta_2$. As a result, $\AC_\U(\U)$ with parametrization $\bTheta_3$ must compute $\pr(\e_1\cond\U, \e_2)$ and remain deterministic. We finally evaluate $\AC_\U$ bottom-up while setting all indicates to $1$ and replacing every $+$-node in $\AC_\U$ with a $\max$-node. The R-MAP probability is returned at the root. 
\shrink{
The pseudocode for this procedure is provided in Algorithm ???.

\[
    \val(n) = \begin{cases}
        (\one{x \sim \e_1\e_2}, \one{x \sim \e_2}) & \text{if $n$ is an indicator $\lambda_x$} \\
        (p, p) & \text{if $n$ is a parameter $p$} \\
        \hat{\prod}_{c \in \child(n)} \val(c) & \text{if $n$ is a product node} \\
        \hat{\sum}_{c \in \child(n)} \val(c) & \text{if $n$ is a sum node} \\
    \end{cases}
\]
}
We now have the following result which follows from the above discussion.
\begin{theorem}
    Consider a decomposable and smooth $\AC(\X)$ and let $\U \subseteq \X$. If $\AC$ is $\U$-deterministic, then it supports R-MAP with target variables $\U$ in time linear in the AC size.
    \label{thm:rmap}
\end{theorem}

\section{Empirical Results}
We next apply the AC algorithm for R-MAP to solving unit selection problems on random SCMs, and compare its time and complexity to the VE algorithm for R-MAP --- VE\_RMAP in~\cite{huang2023algorithm}. Our algorithm, ACE\_RMAP, is implemented on top of the {\tt ACE} system\footnote{\url{http://reasoning.cs.ucla.edu/ace/}} which we use to compile the objective model into a decision-AC, and then evaluate the AC by the two-pass traversal as described in the previous section. The task is to find optimal units for the benefit function shown in Equation~\ref{eqn:benefit-ang}. 
We next describe the procedure used to generate problem instances.

\paragraph{Problem Generation} We generate random SCM $G$ according to the method in ~\cite{han2022on} and carefully choose unit variables $\U$, outcome variable $Y$, and action variable $X$ of the benefit function $L$ such that $\U, X, Y$ respect meaningful causal relationships. We first generate a DAG $G_0$ with $n$ binary nodes and a maximum number of parents $p$. We then convert $G_0$ into SCM $G$ by adding a unique root parent for each internal node in $G_0$. The resulting $G$ tends to have many roots, which is meant to mimic the structure of SCMs commonly used for counterfactual reasoning. We randomly choose a leaf as outcome $Y$ and an ancestor (cause) of $Y$ as action $X$.
We choose unit variables $\U$ by randomly picking 
50\% of the variables satisfying the following constraints: a unit variable $U$ must be a common cause (ancestor) of $X, Y$ and $U$ remains a cause (ancestor) of $Y$ after removing all edges incoming into $X$.

For each problem instance $\langle G, \U, X, Y \rangle$, we construct an objective model $G'$ by composing triplet-models $G^1,\ldots, G^4,$ one for each component in the benefit function, as discussed earlier; see \cite{huang2023algorithm} for more details. We then run ACE\_RMAP on $G'$ to find the optimal units and compare its results against VE\_RMAP implemented in {\tt NumPy}. The computation for each instance is given 10 minutes to complete. For each problem instance, we report the execution time of ACE\_RMAP and VE\_RMAP. We also report the size of the circuit generated by ACE (ac\_size), and the total size of all factors generated by VE  (ve\_size), which are two comparable parameters that measure the total number of arithmetic operations required by ACE\_RMAP and VE\_RMAP, respectively. We also report the approximate constrained treewidth (tw) of $G'$, which is computed using minfill heuristics~\cite{kjaerulff1990triangulation} to find a constrained elimination order for VE\_RMAP; see~\cite{huang2023algorithm} for the need of a constrained order. 

\paragraph{Results} We used $n \in \{10, 15, 20, 25, 30, 35, 40\}$ and $p=6$ which results in SCM with $n'\simeq2n$ nodes and objective model with $\leq 12n$ nodes. For each $n$, we generate 25 instances and report the average statistics in Table~\ref{tab:benefit-function-exp}. We highlight the patterns from the statistics. First, as $n$ increases, the time and the number operations of  VE\_RMAP grows exponentially and becomes impractical after $n > 30$ (running out of memory\footnote{NumPy does not support ndarrays with $>32$ dimensions.}). This is predicted as VE\_RMAP is purely structure-based and must be exponential in the constrained treewidth (tw). Second, ACE\_RMAP is much more efficient than VE\_RMAP, leading to orders-of-magnitude speedups as a result of exploiting the high-degree of local (parameteric) structure in the objective model (e.g., 0/1 parameters, context-specific independence, parameter equality). This enables ACE\_RMAP to support very large and dense models (with $tw > 30$) that are normally out-of-reach if such parametric structure is not exploited.

\section{Conclusion}
Our main contribution in this paper is a result which shows how we can solve the Reverse-MAP (R-MAP) problem efficiently using ACs. R-MAP is a variant on the classical MAP problem which arises in causal reasoning, particularly the unit selection problem. We evaluated our AC-based method for solving R-MAP (and the unit selection problem) against the state-of-the-art method based on variable elimination, showing orders-of-magnitude speedups. This is a substantial step forward in making causal reasoning more scalable.

\paragraph{Acknowledgement} This work has been partially supported by ONR grant N000142212501.

\bibliography{bib/reference, bib/references, bib/references2, bib/refs}

\begin{thebibliography}{}

\bibitem[\protect\citeauthoryear{Balke and Pearl}{1994}]{aaai/BalkeP94}
Balke, A., and Pearl, J.
\newblock 1994.
\newblock Probabilistic evaluation of counterfactual queries.
\newblock In {\em {AAAI}},  230--237.
\newblock {AAAI} Press / The {MIT} Press.

\bibitem[\protect\citeauthoryear{Balke and Pearl}{1995}]{uai/BalkeP95}
Balke, A., and Pearl, J.
\newblock 1995.
\newblock Counterfactuals and policy analysis in structural models.
\newblock In {\em {UAI}},  11--18.
\newblock Morgan Kaufmann.

\bibitem[\protect\citeauthoryear{Bareinboim \bgroup et al\mbox.\egroup }{2021}]{Bareinboim20211OP}
Bareinboim, E.; Correa, J.~D.; Ibeling, D.; and Icard, T.~F.
\newblock 2021.
\newblock On {Pearl}'s hierarchy and the foundations of causal inference.
\newblock Technical Report, R-60, Colombia University.

\bibitem[\protect\citeauthoryear{Chan and Darwiche}{2006}]{ChanD06}
Chan, H., and Darwiche, A.
\newblock 2006.
\newblock On the robustness of most probable explanations.
\newblock In {\em Proceedings of the 22nd Conference in Uncertainty in Artificial Intelligence (UAI)}.

\bibitem[\protect\citeauthoryear{Chavira and Darwiche}{2005}]{Chavira.Darwiche.Ijcai.2005}
Chavira, M., and Darwiche, A.
\newblock 2005.
\newblock Compiling bayesian networks with local structure.
\newblock In {\em Proceedings of the 19th International Joint Conference on Artificial Intelligence (IJCAI)},  1306--1312.

\bibitem[\protect\citeauthoryear{Chavira and Darwiche}{2007}]{Chavira.Darwiche.Ijcai.2007}
Chavira, M., and Darwiche, A.
\newblock 2007.
\newblock Compiling {B}ayesian networks using variable elimination.
\newblock In {\em Proceedings of the 20th International Joint Conference on Artificial Intelligence (IJCAI)},  2443--2449.

\bibitem[\protect\citeauthoryear{Chavira and Darwiche}{2008}]{Chavira.Darwiche.Aij.2008}
Chavira, M., and Darwiche, A.
\newblock 2008.
\newblock On probabilistic inference by weighted model counting.
\newblock {\em Artificial Intelligence} 172(6--7):772--799.

\bibitem[\protect\citeauthoryear{Chavira, Darwiche, and Jaeger}{2006}]{ChaviraDJ06}
Chavira, M.; Darwiche, A.; and Jaeger, M.
\newblock 2006.
\newblock Compiling relational bayesian networks for exact inference.
\newblock {\em Int. J. Approx. Reasoning} 42(1-2):4--20.

\bibitem[\protect\citeauthoryear{Choi and Darwiche}{2017}]{ChoiDarwiche17}
Choi, A., and Darwiche, A.
\newblock 2017.
\newblock On relaxing determinism in arithmetic circuits.
\newblock In {\em Proceedings of the Thirty-Fourth International Conference on Machine Learning (ICML)},  825--833.

\bibitem[\protect\citeauthoryear{Choi, Friedman, and Van~den Broeck}{2022}]{choi2022solving}
Choi, Y.; Friedman, T.; and Van~den Broeck, G.
\newblock 2022.
\newblock Solving marginal map exactly by probabilistic circuit transformations.
\newblock In {\em International Conference on Artificial Intelligence and Statistics},  10196--10208.
\newblock PMLR.

\bibitem[\protect\citeauthoryear{Choi, Vergari, and Van~den Broeck}{2020}]{choi2020probabilistic}
Choi, Y.; Vergari, A.; and Van~den Broeck, G.
\newblock 2020.
\newblock Probabilistic circuits: A unifying framework for tractable probabilistic models.
\newblock {\em UCLA. URL: http://starai. cs. ucla. edu/papers/ProbCirc20. pdf}.

\bibitem[\protect\citeauthoryear{Darwiche}{2003}]{Darwiche03}
Darwiche, A.
\newblock 2003.
\newblock A differential approach to inference in {B}ayesian networks.
\newblock {\em Journal of the ACM (JACM)} 50(3):280--305.

\bibitem[\protect\citeauthoryear{Darwiche}{2004}]{darwiche2004new}
Darwiche, A.
\newblock 2004.
\newblock New advances in compiling cnf to decomposable negation normal form.
\newblock In {\em Proc. of ECAI},  328--332.
\newblock Citeseer.

\bibitem[\protect\citeauthoryear{Darwiche}{2009}]{DarwicheBook09}
Darwiche, A.
\newblock 2009.
\newblock {\em Modeling and Reasoning with Bayesian Networks}.
\newblock Cambridge University Press.

\bibitem[\protect\citeauthoryear{Darwiche}{2021}]{faia/Darwiche21}
Darwiche, A.
\newblock 2021.
\newblock Tractable boolean and arithmetic circuits.
\newblock In {\em Neuro-Symbolic Artificial Intelligence}, volume 342 of {\em Frontiers in Artificial Intelligence and Applications}. {IOS} Press.
\newblock  146--172.

\bibitem[\protect\citeauthoryear{Dechter and Rish}{2003}]{dechter2003mini}
Dechter, R., and Rish, I.
\newblock 2003.
\newblock Mini-buckets: A general scheme for bounded inference.
\newblock {\em Journal of the ACM (JACM)} 50(2):107--153.

\bibitem[\protect\citeauthoryear{Halpern}{2000}]{halpern2000axiomatizing}
Halpern, J.~Y.
\newblock 2000.
\newblock Axiomatizing causal reasoning.
\newblock {\em Journal of Artificial Intelligence Research} 12:317--337.

\bibitem[\protect\citeauthoryear{Han, Chen, and Darwiche}{2022}]{han2022on}
Han, Y.; Chen, Y.; and Darwiche, A.
\newblock 2022.
\newblock On the complexity of counterfactual reasoning.
\newblock In {\em A causal view on dynamical systems, NeurIPS 2022 workshop}.

\bibitem[\protect\citeauthoryear{Huang and Darwiche}{2007}]{huang2007language}
Huang, J., and Darwiche, A.
\newblock 2007.
\newblock The language of search.
\newblock {\em Journal of Artificial Intelligence Research} 29:191--219.

\bibitem[\protect\citeauthoryear{Huang and Darwiche}{2023}]{huang2023algorithm}
Huang, H., and Darwiche, A.
\newblock 2023.
\newblock An algorithm and complexity results for causal unit selection.
\newblock In {\em 2nd Conference on Causal Learning and Reasoning}.

\bibitem[\protect\citeauthoryear{Huang \bgroup et al\mbox.\egroup }{2006}]{huang2006solving}
Huang, J.; Chavira, M.; Darwiche, A.; et~al.
\newblock 2006.
\newblock Solving map exactly by searching on compiled arithmetic circuits.
\newblock In {\em AAAI}, volume~6,  3--7.

\bibitem[\protect\citeauthoryear{Kj{\ae}rulff}{1990}]{kjaerulff1990triangulation}
Kj{\ae}rulff, U.~B.
\newblock 1990.
\newblock Triangulation of graphs-algorithms giving small total state space.

\bibitem[\protect\citeauthoryear{Koller and Friedman}{2009}]{PGMbook}
Koller, D., and Friedman, N.
\newblock 2009.
\newblock {\em Probabilistic Graphical Models - Principles and Techniques}.
\newblock {MIT} Press.

\bibitem[\protect\citeauthoryear{Li and Pearl}{2019}]{ijcai/LiP19}
Li, A., and Pearl, J.
\newblock 2019.
\newblock Unit selection based on counterfactual logic.
\newblock In {\em {IJCAI}},  1793--1799.
\newblock ijcai.org.

\bibitem[\protect\citeauthoryear{Li and Pearl}{2022a}]{li2022unit_ab}
Li, A., and Pearl, J.
\newblock 2022a.
\newblock Unit selection: Case study and comparison with a/b test heuristic.
\newblock {\em arXiv preprint arXiv:2210.05030}.

\bibitem[\protect\citeauthoryear{Li and Pearl}{2022b}]{aaai/LiP22}
Li, A., and Pearl, J.
\newblock 2022b.
\newblock Unit selection with causal diagram.
\newblock In {\em {AAAI}},  5765--5772.
\newblock {AAAI} Press.

\bibitem[\protect\citeauthoryear{Li and Pearl}{2022c}]{corr/LiP22a}
Li, A., and Pearl, J.
\newblock 2022c.
\newblock Unit selection with nonbinary treatment and effect.
\newblock {\em CoRR} abs/2208.09569.

\bibitem[\protect\citeauthoryear{Li \bgroup et al\mbox.\egroup }{2022}]{li2022unit}
Li, A.; Jiang, S.; Sun, Y.; and Pearl, J.
\newblock 2022.
\newblock Unit selection: Learning benefit function from finite population data.
\newblock {\em arXiv e-prints}  arXiv--2210.

\bibitem[\protect\citeauthoryear{Park and Darwiche}{2004}]{jair/ParkD04}
Park, J.~D., and Darwiche, A.
\newblock 2004.
\newblock Complexity results and approximation strategies for {MAP} explanations.
\newblock {\em J. Artif. Intell. Res.} 21:101--133.

\bibitem[\protect\citeauthoryear{Pearl and Mackenzie}{2018}]{pearl18}
Pearl, J., and Mackenzie, D.
\newblock 2018.
\newblock {\em The Book of Why: The New Science of Cause and Effect}.
\newblock Basic Books.

\bibitem[\protect\citeauthoryear{Pearl}{1988}]{Pearl88b}
Pearl, J.
\newblock 1988.
\newblock {\em Probabilistic Reasoning in Intelligent Systems: Networks of Plausible Inference}.
\newblock MK.

\bibitem[\protect\citeauthoryear{Pearl}{2000}]{pearl00b}
Pearl, J.
\newblock 2000.
\newblock {\em Causality}.
\newblock Cambridge University Press.

\bibitem[\protect\citeauthoryear{Peters, Janzing, and Schölkopf}{2017}]{PetersBook}
Peters, J.; Janzing, D.; and Schölkopf, B.
\newblock 2017.
\newblock {\em Elements of Causal Inference: Foundations and Learning Algorithms}.
\newblock MIT Press.

\bibitem[\protect\citeauthoryear{Pipatsrisawat and Darwiche}{2009}]{pipatsrisawat2009new}
Pipatsrisawat, K., and Darwiche, A.
\newblock 2009.
\newblock A new d-dnnf-based bound computation algorithm for functional e-majsat.
\newblock In {\em Twenty-First International Joint Conference on Artificial Intelligence}.
\newblock Citeseer.

\bibitem[\protect\citeauthoryear{Poon and Domingos}{2011}]{PoonD11}
Poon, H., and Domingos, P.~M.
\newblock 2011.
\newblock Sum-product networks: {A} new deep architecture.
\newblock In {\em UAI},  337--346.

\bibitem[\protect\citeauthoryear{Spirtes, Glymour, and Scheines}{2000}]{SpirtesBook}
Spirtes, P.; Glymour, C.; and Scheines, R.
\newblock 2000.
\newblock {\em Causation, Prediction, and Search, Second Edition}.
\newblock Adaptive computation and machine learning. {MIT} Press.

\end{thebibliography}
\bibliographystyle{flairs}

\end{document}